\documentclass{article}

\usepackage[preprint]{neurips_2026}
\usepackage[utf8]{inputenc}
\usepackage[T1]{fontenc}
\usepackage{hyperref}
\usepackage{url}
\usepackage{booktabs}
\usepackage{amsfonts}
\usepackage{amsmath,amssymb,amsthm,mathtools}
\usepackage{nicefrac}
\usepackage{microtype}
\usepackage{xcolor}
\usepackage{enumitem}
\usepackage{multirow}
\usepackage{graphicx}
\usepackage{subcaption}
\usepackage{array}
\usepackage{tabularx}
\usepackage{makecell}
\usepackage{array}
\usepackage{rotating}
\usepackage{caption}
\usepackage{placeins}
\usepackage{float}

\title{Geometry-Conditioned Fixed-Scaffold Encoders for Time-Warp Robust Sequence Retrieval}

\author{%
  Cassandra Yang\\
  New York University\\
  \texttt{cassandrayang@nyu.edu}
  \And
  Yufan Tang\\
  Fudan University\\
  \texttt{23307130372@m.fudan.edu.cn}
}

\newtheorem{theorem}{Theorem}
\newtheorem{proposition}{Proposition}

\newtheorem{corollary}{Corollary}

\newcommand{\R}{\mathbb{R}}

\newcommand{\norm}[1]{\left\lVert #1 \right\rVert}
\newcommand{\abs}[1]{\left\lvert #1 \right\rvert}

\newcommand{\Aff}{\mathcal{A}}
\newcommand{\Proj}{\Pi_{\Aff}}

\begin{document}
\maketitle
\vspace{-1.0em}

\begin{abstract}
Embedding-based retrieval is attractive for long sequence collections because each item can be encoded once and searched by nearest-neighbor ranking. The difficulty is that the objects being indexed are often observed under a noncanonical clock: cardiac cycles stretch with rate, speech changes with tempo, and sensor traces reach comparable states at different speeds. This paper studies a specific source of instability in patch-based encoders for this regime.  If patch boundaries are chosen from signal geometry, then the tokenization can change under the same temporal deformation that the representation is expected to tolerate. We propose \textbf{GeoPatch}, a fixed-scaffold patch encoder that keeps token support independent of geometry and uses slope, curvature, acceleration, affine-residual, and confidence descriptors only as continuous conditioning variables. The design turns boundary variation into feature modulation: geometry can change the embedding through a controlled pathway, but it cannot change the number, order, or support of local tokens. We formalize this distinction through a mechanism-level stability analysis that separates boundary drift, affine timing variation, confidence-weighted geometry perturbation, and retrieval-margin effects. The same local tokens support global embedding retrieval and late-interaction scoring, so the scoring rule can be matched to the evaluation protocol.  Across ECG, speech, and multivariate time-series retrieval tasks, GeoPatch improves early-rank retrieval under timing variation while exposing a clear trade-off between local surface matching and strict non-overlap retrieval.
\end{abstract}
\vspace{0.2em}

\section{Introduction}
Sequence retrieval often asks a model to identify the same local event after the clock has changed.  A normal beat and a faster beat may share morphology, a spoken prompt may preserve phonetic content under different tempo, and a motion trace may pass through the same configuration at a shifted phase.  Pairwise alignment methods handle this variability by optimizing correspondences, as in DTW, soft-DTW, generalized time warping, differentiable DTW, and diffeomorphic temporal alignment \citep{muller2007dtw,cuturi2017soft,zhou2012gtw,xu2023decdtw,weber2019dtan,weber2023regularization,weber2025timepoint}.  Large-gallery retrieval imposes a different constraint: each sequence should be embedded once, indexed, and compared without solving a full alignment problem against every candidate.  Recent time-series representation methods and patch Transformers make this embedding view increasingly effective \citep{yue2022ts2vec,eldele2021tstcc,tonekaboni2021tnc,lee2024softclt,nie2023patchtst}, but they do not directly address how local tokenization behaves under temporal deformation.

The tokenization step matters because patch boundaries are discrete.  When a patch starts at a fixed grid location, a time warp perturbs the values inside the patch.  When a patch starts at a peak, corner, curvature extremum, or learned cut point, the same warp can also move the boundary.  The second case is qualitatively different: the encoder receives a different collection of local domains before any learned invariance can act.  This creates a conflict between two otherwise reasonable goals.  Geometry is useful evidence for matching local morphology, but using geometry to choose token boundaries makes the input units depend on the nuisance variation.

We call this conflict the segmentation-invariance paradox.  It is not a claim that adaptive segmentation is always inferior; adaptive boundaries can be beneficial when the boundary rule is stable and the downstream target is aligned with the segmentation. However, in warp-tolerant retrieval, a geometry-dependent boundary is an additional random variable whose stability must be established.  If that stability is absent, an encoder is asked to identify equivalent content after the content has been partitioned into non-equivalent tokens.

GeoPatch addresses this instability at the patch interface.  It extracts patches from a fixed scaffold and passes local geometry as conditioning rather than as a segmentation rule.  Each patch therefore has two roles: its support defines a comparable unit across observations, while its geometric descriptor describes the local state inside that unit.  The descriptor includes slope, curvature, acceleration, affine-residual summaries, and confidence.  These quantities modulate the patch embedding continuously through a lightweight geometry branch, but they cannot alter the token support.  The resulting local tokens can be pooled into a single vector or retained for late-interaction retrieval in the style of token-level MaxSim scoring \citep{khattab2020colbert}.

We thus develop threefold contributions on this novel retrieval mechanism.  First, we introduce a fixed-scaffold, geometry-conditioned encoder with validity masks, confidence weighting, decorrelation, and protocol-aware aggregation. Second, we give a mechanism-level analysis showing how the architecture removes a discrete source of drift and converts the remaining geometry variation into a bounded perturbation of the embedding. Third, we evaluate the method on ECG, speech, and UEA retrieval tasks, reporting both early-rank and full-ranking metrics and separating standard retrieval from strict non-overlap evaluation.

\section{Related Work}
\paragraph{Temporal alignment.}
DTW and soft-DTW define similarity through monotone correspondence paths, while generalized time warping and optimal-transport warping extend the correspondence objective to richer constraints and modalities \citep{muller2007dtw,cuturi2017soft,zhou2012gtw,latorre2023otw}.  Differentiable alignment layers and learned temporal alignment models, including D3TW, DecDTW, DTAN, RF-DTAN, and TimePoint, make correspondence search more compatible with learning or faster inference \citep{chang2019d3tw,xu2023decdtw,weber2019dtan,weber2023regularization,weber2025joint,weber2025timepoint}.  GeoPatch targets the same nuisance variable but deliberately avoids estimating an alignment path during gallery search.  Its contribution is to preserve a stable local indexing substrate while allowing geometry to influence the representation continuously.

\paragraph{Time-series encoders and contrastive representation learning.}
Convolutional baselines, random-kernel transforms, and patch Transformers provide strong archive-level time-series representations \citep{wang2017resnet,ismail2020inceptiontime,dempster2020rocket,dempster2021minirocket,nie2023patchtst}.  Self-supervised and contrastive models such as CPC, TS-TCC, TNC, TS2Vec, CoST, SoftCLT, and transformer-based multivariate encoders learn temporal invariances from augmentations, neighborhoods, or soft positive assignments \citep{oord2018contrastive,eldele2021tstcc,tonekaboni2021tnc,yue2022ts2vec,woo2022cost,lee2024softclt,zerveas2021transformer}.  SoftCLT is especially relevant because it relaxes both data-space and timestamp-level positive structure.  GeoPatch is orthogonal to these objectives: it changes how tokens are constructed and conditioned, and it can be trained with standard supervised contrastive losses \citep{khosla2020supervised}.  CLOCS motivates the need for split-aware cardiac evaluation because local temporal and patient-level structure can otherwise be confounded \citep{kiyasseh2021clocs}.

\paragraph{Invariant architectures and geometry.}
Hard-coded invariance methods show that known nuisance transformations can be controlled architecturally rather than learned entirely from data \citep{germain2025invconv}.  GeoPatch follows this principle locally but does not remove all geometry.  Curvature and acceleration may be discriminative, while shifts and local linear rate changes may be nuisance.  The model therefore residualizes affine timing components, weights descriptors by confidence, and injects residual geometry as a continuous condition.  This is a selective invariance strategy: the scaffold is invariant to geometry, but the token content remains geometry-aware.

\paragraph{Aggregation and retrieval protocols.}
Patch tokens can be aggregated by set functions, attention-based set models, or ordered Transformers depending on the target relation \citep{zaheer2017deepsets,lee2019set,vaswani2017attention}.  Late-interaction retrieval, popularized by ColBERT, delays token-level matching until after encoding and often improves early precision \citep{khattab2020colbert}.  GeoPatch uses the same principle for sequence patches.  However, token-level matching is not intrinsically better under every protocol: alignable-video retrieval work makes clear that the definition of a valid match shapes the appropriate metric and scorer \citep{dave2024sync}.  We therefore report standard and strict retrieval separately and treat MaxSim as an operating point rather than as a universally preferred aggregation rule.

\section{Problem Setup and Mechanism}
\label{sec:mechanism}
Let $x: [0,1] \to \mathbb{R}^d$ denote a canonical source signal. An observation $y$, subject to a monotone temporal deformation, is formalized through the generative model:
\begin{equation}
    y(t) = A(t) \odot x(\sigma(t)) + \epsilon(t),
    \label{eq:gen}
\end{equation}
where $\sigma$ defines the non-linear observed-to-canonical clock mapping, $A(t)$ represents channel-wise amplitude modulation, and $\epsilon(t)$ encapsulates additive noise.

The central distinction is whether geometry enters before or after discretization.  Let a geometry-dependent segmentation rule produce a patch family
\begin{equation}
    \mathcal P(y)=\{\Omega_1(y),\ldots,\Omega_{K(y)}(y)\},
\end{equation}
and let a patch representation be written abstractly as
\begin{equation}
    z_{\mathrm{seg}}(y)=\rho\left(\{f(y|_{\Omega_k(y)})\}_{k=1}^{K(y)}\right).
    \label{eq:adaptive_rep}
\end{equation}
A temporal warp then affects both the values inside each patch and the patch family itself.  For two content-matched observations $y_A$ and $y_B$, the discrepancy admits a boundary--feature decomposition obtained by inserting the cross-term that applies $y_B$'s patch family to the underlying signal $y_A$:
\begin{align}
    \|z_{\mathrm{seg}}(y_A)-z_{\mathrm{seg}}(y_B)\|
    &\le
    \underbrace{\|R(\mathcal P(y_A),y_A)-R(\mathcal P(y_B),y_A)\|}_{\text{boundary drift}}
    +
    \underbrace{\|R(\mathcal P(y_B),y_A)-R(\mathcal P(y_B),y_B)\|}_{\text{feature drift}},
    \label{eq:drift_decomp}
\end{align}
where $R(\mathcal P,y)$ denotes the functional that extracts patch values $\{y|_{\Omega}\}_{\Omega\in\mathcal P}$ from the continuous signal $y$ under patch family $\mathcal P$ and then applies patch encoding and aggregation.  Standard Lipschitz arguments control the second term only after the domains are fixed. The first term is a discrete change in the objects being encoded and is not controlled by smoothness of $f$ or $\rho$ alone.

GeoPatch removes this first term by construction.  It fixes a scaffold $\mathcal S=\{I_k\}_{k=1}^K$, with $I_k=[t_k,t_k+\ell]$ by default, and computes
\begin{equation}
    y_k=y|_{I_k},\qquad g_k=G(y_k),\qquad m_k\in\{0,1\}.
\end{equation}
The descriptor $G$ contains slope, curvature, acceleration, affine-residual summaries, and confidence, but it does not define $I_k$.  The representation factorizes as
\begin{equation}
    z_{\mathrm{GeoPatch}}(y)=\rho_\theta\left(\left\{F_\theta\big(\psi_\theta(y_k),\gamma_\theta(g_k),m_k\big)\right\}_{k=1}^K\right).
    \label{eq:GeoPatch_factorization}
\end{equation}
The scaffold is common across observations, so geometry-induced variation enters through $g_k$ and the modulation network $\gamma_\theta$, not through token support.  This is tested in the experiments and analyzed in Section~\ref{sec:stability}.

\section{GeoPatch Architecture}
\label{sec:GeoPatch}
\subsection{Overview}
GeoPatch transforms an observed signal $y$ into a normalized sequence embedding $z$ through a dual-stream pipeline of patch extraction and geometric conditioning:
\begin{equation}
    y \longrightarrow \{y_k,g_k,m_k\}_{k=1}^K
    \longrightarrow \{z_k\}_{k=1}^K
    \longrightarrow z \quad \text{or} \quad \{z_k\}_{k=1}^K .
    \label{eq:pipeline}
\end{equation}
The architecture follows the factorization in Eq.~\eqref{eq:GeoPatch_factorization}: stabilize the patch support, encode local content, encode local geometry, fuse the two streams, and aggregate according to the retrieval protocol.  Figure~\ref{fig:GeoPatch} gives the complete pipeline.

\begin{figure}[t]
    \centering
    \IfFileExists{geopatch_framework.png}{%
      \includegraphics[width=\textwidth]{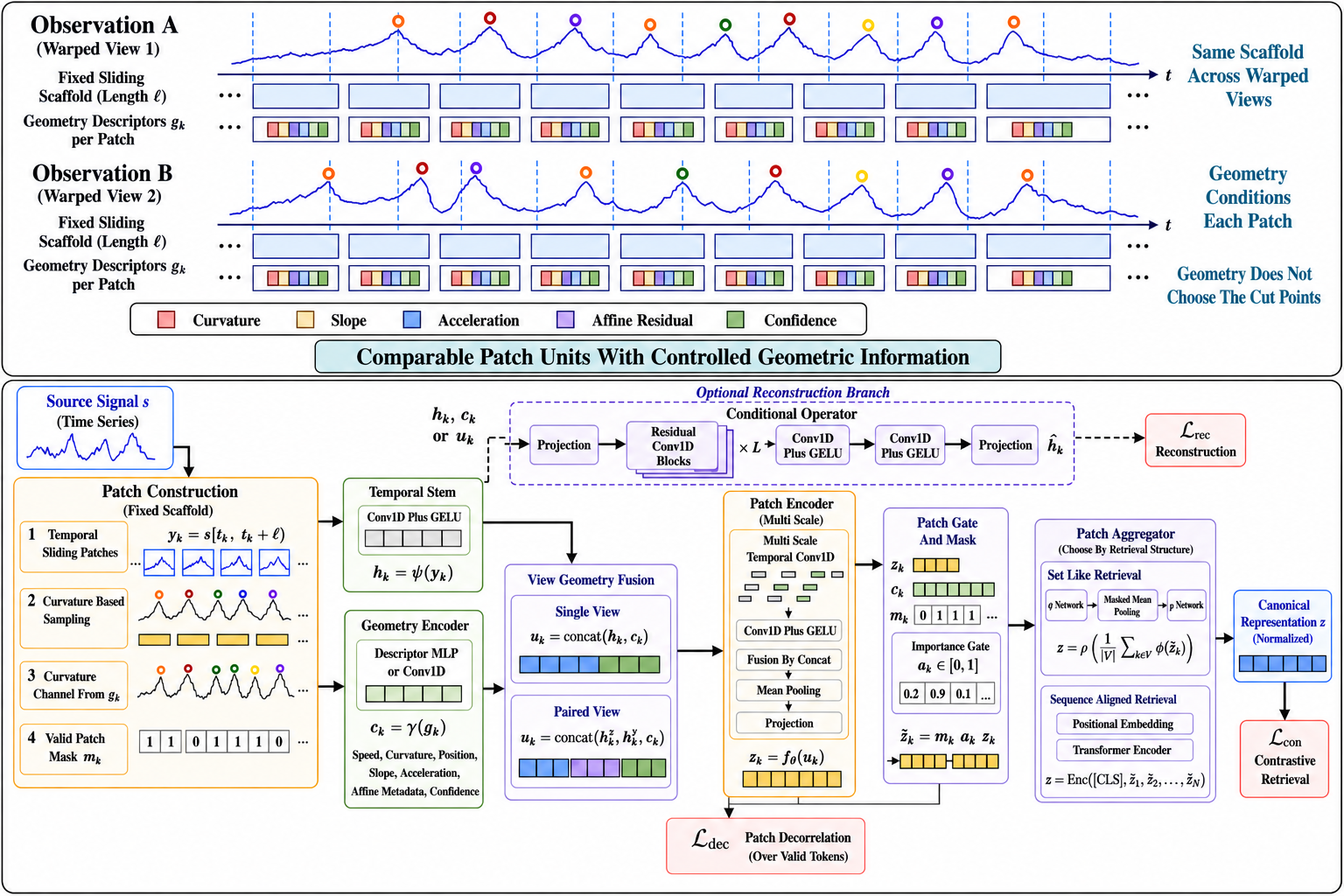}%
    }{%
      \fbox{\parbox[c][1.35in][c]{0.98\linewidth}{\centering\footnotesize GeoPatch architecture figure}}%
    }
     \caption{\textbf{The GeoPatch framework architecture.} Raw sequences are mapped to a structural patch scaffold in parallel with the extraction of continuous geometric descriptors (e.g., curvature and slope). A temporal stem and a geometry module independently process these streams before a fusion block integrates them into geometry-conditioned local tokens. To handle sequence noise, a validity mask gates unreliable or padding patches. The resulting tokens are either aggregated into a global embedding or cached for late-interaction retrieval.}
    \label{fig:GeoPatch}
\end{figure}

\subsection{Patch construction and descriptors}
Patches are extracted on a fixed scaffold $I_k=[t_k,t_k+\ell]$ with stride $\Delta$, giving $y_k=y[t_k:t_k+\ell]$. The geometry descriptor $g_k$ contains local first- and second-order finite-difference statistics, curvature summaries, affine-residual features, and confidence indicators.  A geometry encoder maps these descriptors to $c_k=\gamma_\theta(g_k)$.  The scaffold is fixed with respect to the descriptor: geometry is observed inside the token, not used to define the token.

\subsection{Geometry-conditioned patch tokens}
A temporal stem produces $h_k=\psi_\theta(y_k)$.  Content and geometry are fused before local encoding,
\begin{equation}
    u_k = \mathrm{Fuse}_\theta(h_k,c_k), \qquad z_k=f_\theta(u_k),
    \label{eq:token}
\end{equation}
where $\mathrm{Fuse}_\theta$ may be implemented by concatenation followed by projection, gated modulation, or feature-wise affine conditioning. This architecture is modular; $f_\theta$ can be implemented as a multiscale CNN for clinical morphology or an attention-based block for sequence-aligned data, depending on the domain requirements.

\subsection{Validity, gating, and decorrelation}
The gated token $\tilde z_k = m_k a_k z_k$ uses a learned gate
\begin{equation}
    a_k = m_k\,\sigma\!\left(W_2\,\mathrm{GELU}(W_1\,\xi_k+b_1)+b_2\right) , \label{eq:gate}
\end{equation}
where $\xi_k\in\mathbb R^4$ stacks $\log(1+\|\nabla^j r_k\|_1/(\ell-j))$ for $j=0,1,2$ and $\log(1+\bar\kappa_k)$, with $\bar c_k$ multiplying the $j{=}2$ and curvature entries; $r_k$ is the within-patch affine residual, $\bar c_k\in[0,1]$ is mean per-position validity, $\bar\kappa_k$ is mean patch curvature, and $\ell$ is the patch length. Initialized with $W_2=0, b_2=-2$, so $a_k\approx 0.12\,m_k$.

To prevent local representation collapse, valid tokens are regularized with a patch decorrelation loss:
\begin{equation}
    \mathcal L_{\mathrm{dec}} = \frac{1}{B}\sum_{b=1}^{B}\frac{\mathbb 1[|\mathcal M_b|>1]}{|\mathcal M_b|(|\mathcal M_b|-1)} \sum_{\substack{i,j\in\mathcal M_b\\ i\ne j}}\left|\frac{\tilde z_{b,i}^{\top}\tilde z_{b,j}}{\norm{\tilde z_{b,i}}_2\norm{\tilde z_{b,j}}_2+\varepsilon}\right| ,
    \label{eq:decorrelation}
\end{equation}
where $\mathcal M_b$ denotes the valid-patch index set for sequence $b$ within the batch, and $\varepsilon$ ensures numerical stability.

\subsection{Aggregation and retrieval scoring}
For single-vector retrieval, valid local tokens are aggregated into a normalized embedding. A set-style aggregator is applied where the presence of local morphology may dominate absolute order:
\begin{equation}
    z=\rho_\theta\left(
    \frac{\sum_{k=1}^K m_k\,\phi_\theta(\tilde z_k)}{\sum_{k=1}^K m_k+\varepsilon}
    \right),
    \label{eq:deepsets}
\end{equation}
while an ordered aggregator is used for sequence-aligned tasks, where temporal progression matters:
\begin{equation}
    z = \mathrm{Enc}_\theta\left([\mathrm{cls}],\tilde z_1+e(p_1),\ldots,\tilde z_K+e(p_K)\right)_{\mathrm{cls}} ,
    \label{eq:transformeragg}
\end{equation}
where $e(p_k)$ denotes a learned positional embedding for the patch index $p_k$. For late interaction, cached patch tokens are compared after encoding:
\begin{equation}
    S(q,c)=\sigma\left(\{\cos(z_i^{(q)},z_j^{(c)})\}_{i\in\mathcal M_q,j\in\mathcal M_c}\right),
    \label{eq:lateinteraction}
\end{equation}
where $\sigma$ is MaxSim or a smooth log-sum-exp bridge. This scorer follows the same late-interaction principle as ColBERT \citep{khattab2020colbert}: local matches are preserved after encoding rather than collapsed immediately into one vector. Because computing dense pairwise similarities across all sequence patches scales poorly to massive galleries, this late-interaction step operates strictly as a re-ranker on a reduced candidate shortlist generated via the global embedding $z$. 

\subsection{Training objective}
The training objective is
\begin{equation}
    \mathcal L
    = \lambda_{\mathrm{con}}\mathcal L_{\mathrm{con}}
    + \lambda_{\mathrm{rec}}\mathcal L_{\mathrm{rec}}
    + \lambda_{\mathrm{dec}}\mathcal L_{\mathrm{dec}}
    + \lambda_{\mathrm{task}}\mathcal L_{\mathrm{task}},
    \label{eq:objective}
\end{equation}
where $\mathcal L_{\mathrm{con}}$ is a supervised contrastive retrieval loss \citep{khosla2020supervised,oord2018contrastive}, $\mathcal L_{\mathrm{rec}}$ is optional reconstruction, $\mathcal L_{\mathrm{dec}}$ is Eq.~\eqref{eq:decorrelation}, and $\mathcal L_{\mathrm{task}}$ denotes optional classification or ranking losses.  Unless otherwise stated, retrieval uses normalized embeddings with the self-match diagonal removed.

\section{Stability Analysis}
\label{sec:stability}

The fixed-scaffold construction changes the stability problem from one of token support to one of feature modulation. With adaptive patching, a time deformation can change the support, ordering, or cardinality of the patch collection before the encoder is applied.  With a fixed scaffold, the patch domains remain shared across examples, and perturbations enter through patch content and geometry descriptors. This section formalizes that separation and connects it to retrieval behavior.

\paragraph{Boundary drift under adaptive patching.}
Let $\mathcal P(y)$ denote the patch collection induced by an input-dependent segmentation rule. As established in Eq.~\eqref{eq:drift_decomp}, the representation discrepancy between two time-warped observations, $y_A$ and $y_B$, decomposes into two distinct components: boundary drift ($\Delta_{\mathrm{bdry}}$) and feature drift ($\Delta_{\mathrm{feat}}$). 

The feature drift term ($\Delta_{\mathrm{feat}}$) is controlled by the continuity of the encoder and aggregator once the domains are fixed. The boundary drift term ($\Delta_{\mathrm{bdry}}$) is structurally different: if $\mathcal P(y_A)$ and $\mathcal P(y_B)$ differ in support, order, or cardinality, the aggregator receives different semantic units rather than perturbations of a common token set.

\paragraph{Affine residualization removes local clock gauge.}
Within a patch, let $\tau_k\in L^2([0,1])$ denote a local warp restriction and define the affine subspace
\begin{equation}
    \Aff=\{u\mapsto au+b:a,b\in\R\}.
\end{equation}
With $\Proj$ denoting orthogonal projection onto $\Aff$, define
\begin{equation}
    r(\tau_k)=\tau_k-\Proj(\tau_k), \qquad
    \mathcal R(\tau_k)=\frac{r(\tau_k)}{\alpha_{\mathrm{loc}}(\tau_k)},
    \label{eq:affineresid}
\end{equation}
where $\alpha_{\mathrm{loc}}$ is positive and homogeneous.  If $\tau_{\mathrm{new}}=\alpha\tau+\beta$ with $\alpha>0$, then
\begin{equation}
    r(\tau_{\mathrm{new}})=\alpha r(\tau),
    \qquad
    \mathcal R(\tau_{\mathrm{new}})=\mathcal R(\tau)
\end{equation}
whenever $\alpha_{\mathrm{loc}}(\alpha\tau+\beta)=\alpha\alpha_{\mathrm{loc}}(\tau)$.  The residual descriptor is therefore invariant to local shifts and linear rate changes after normalization, while retaining non-affine timing structure.  Geometry conditioning receives this residual component instead of raw timing variation, which aligns the descriptor with local shape rather than absolute clock choice.

\paragraph{Confidence-weighted modulation stability.}
Using the symbols of Section~\ref{sec:GeoPatch}: let $g_k$ denote the per-patch geometry descriptor consumed by the geometry encoder $\gamma_\theta$, let $a_k\in[0,1]$ denote the scalar gate from Eq.~\eqref{eq:gate}, and let $\mathcal M=\{k:m_k=1\}$ denote the valid-patch index set of the single sequence under analysis (the per-sequence analogue of $\mathcal M_q,\mathcal M_c$ in Eq.~\eqref{eq:lateinteraction}).  For a fixed input $y$ and scaffold $\mathcal S$, define the normalized embedding map
\begin{equation}
    T\big((g_k)_{k=1}^K\big)=\hat z\!\left(y,(g_k)_{k=1}^K,(a_k)_{k=1}^K\right).
\end{equation}
Assume that $\gamma_\theta$, the patch encoder $f_\theta$, the aggregator $A_\theta$, and the final normalization are Lipschitz on the bounded operating region with constants $L_\gamma$, $L_f$, $L_A$, $L_N$, and that $\mathcal M$ is unchanged under descriptor perturbations.  Then for any perturbation $(\delta g_k)_{k=1}^K$,
\begin{equation}
    \norm{T\big((g_k+\delta g_k)\big)-T\big((g_k)\big)}_2
    \le
    L_NL_AL_f L_\gamma\cdot
    \frac{1}{|\mathcal M|}\sum_{k\in\mathcal M}a_k\norm{\delta g_k}_2 .
    \label{eq:drift}
\end{equation}
Thus, once support drift is removed, descriptor perturbations affect the sequence embedding only through a weighted continuous path.  The confidence weights scale the contribution of each patch to the drift budget: uncertain descriptors are attenuated, whereas reliable descriptors remain active conditioning variables.  The bound also separates descriptor corruption from evidence removal.  Perturbing curvature, slope, or acceleration changes $u_k$ under a fixed valid set; masking a span changes the observed patch content and may change $\mathcal M$.

\paragraph{Retrieval stability from score margins.}
Let $\hat z_q$ and $\hat z_c$ be unit-normalized query and candidate embeddings.  If perturbations satisfy $\norm{\delta_q}_2\le\eta$ and $\norm{\delta_c}_2\le\eta$, then
\begin{equation}
    \left|
    (\hat z_q+\delta_q)^\top(\hat z_c+\delta_c)
    -
    \hat z_q^\top \hat z_c
    \right|
    \le
    2\eta+\eta^2 .
\end{equation}
A positive candidate remains inside the top-$R$ set whenever its margin over candidates below the cutoff exceeds the maximal score drift.

\paragraph{Late interaction and protocol-dependent ranking.}
Late interaction scores a query-candidate pair through local token similarities,
\begin{equation}
    S(q,c)=\sigma\left(\{\cos(z_i^{(q)},z_j^{(c)})\}_{i\in\mathcal M_q,j\in\mathcal M_c}\right),
\end{equation}
with MaxSim corresponding to a monotone emphasis on the strongest local match.  This scoring rule is well matched to retrieval tasks where local morphology defines relevance.  It also changes the ranking bias: candidates with one highly similar patch can be promoted even when global identity is weaker.  The empirical Standard-vs-Strict analysis therefore follows directly from the scoring rule.  Standard retrieval rewards local surface matches, while strict non-overlap retrieval places greater weight on cross-record identity.

\section{Datasets and Evaluation}
\subsection{Datasets}
We evaluate the method in three retrieval regimes that stress different forms of temporal variability.
\begin{itemize}[leftmargin=1.25em,itemsep=0.15em,topsep=0.15em]
    \item \textbf{Rhythm ECG.}  NSRDB, NSTDB, and MITDB evaluate window-level retrieval under variation in cardiac rate, phase, morphology, and recording context.  This setting is sensitive to protocol leakage: overlapping windows and same-record positives can be easier than cross-record matches, motivating explicit attention to temporal and patient/record structure \citep{kiyasseh2021clocs}.
    \item \textbf{Speech.}  CMU ARCTIC evaluates prompt- and speaker-relevant retrieval under tempo, articulation, and prosodic variation.  Although our task is retrieval rather than recognition, audio Transformer work provides relevant domain context \citep{gong2021ast,gong2022ssast,feng2023flexiast}.
    \item \textbf{UEA multivariate archive.}  Seventeen UEA datasets provide a heterogeneous test of whether fixed-scaffold geometry conditioning transfers beyond a single signal domain, covering physiological, motion, tracking, acoustic, and general sensor sequences.  UEA-mean and UEA-median in Table~\ref{tab:main_retrieval}(a) aggregate over all 17; the robustness sweep in Table~\ref{tab:full_robustness} uses a 7-dataset representative subset spanning these clusters (full list in Appendix~\ref{app:robustness}).
\end{itemize}

\subsection{Evaluation}
Each test sequence is used as a query against a gallery, with embeddings compared by cosine similarity after $\ell_2$ normalization and the self-match diagonal removed.  We report Recall@$R$, MRR, and mAP. Recall@1 measures first-hit behavior—the natural operating point for short-list retrieval—while mAP measures the quality of the entire ranking.

When overlap or same-record ambiguity exists, we distinguish \textbf{Standard} retrieval, which uses the candidate set induced by the dataset construction, from \textbf{Strict} retrieval, which removes overlapping or same-record positives to emphasize cross-record identity. This distinction determines whether a strong local patch match should be rewarded or treated as a near-duplicate; because late interaction directly optimizes local patch similarity, the standard-vs-strict split is part of the experimental design rather than an auxiliary robustness check.

\section{Experiments}
\label{sec:experiments}

We evaluate the GeoPatch framework across three high-variance domains: cardiac rhythm analysis (ECG), acoustic speech retrieval, and heterogeneous multivariate time-series (UEA Archive). Our evaluation focuses on the model's capacity for early-rank precision—the ability to isolate a correct match at the top of a ranked list despite significant non-linear temporal deformations. All experiments were conducted on a multi-node cluster of NVIDIA A800 GPUs, requiring approximately 132 GPU-hours for the reported results and 400 GPU-hours for the full research project.

\subsection{Comparative Retrieval Performance}
\label{subsec:main_results}

Table~\ref{tab:main_retrieval}(a) isolates the architectural contribution by comparing models using standard single-vector retrieval. GeoPatch outperforms all baselines on NSRDB by nearly 20 percentage points in Recall@1. On this dataset, variations in heart rate heavily test a model's ability to match local morphology despite timing shifts. The model ties at the maximum possible score on NSTDB and remains competitive on MITDB and the UEA archive. These results directly support the stability analysis from Section~\ref{sec:stability}. Replacing hard, geometry-dependent patch boundaries with fixed windows and continuous geometric descriptors creates an embedding space that resists severe timing deformations.

\subsection{Stress Testing: Robustness to Deformation and Information Loss}
\label{subsec:robustness}

We examine how specific disruptions affect the learned representations using controlled synthetic perturbations, with results detailed in Table~\ref{tab:full_robustness} and visually summarized in Figure~\ref{fig:ablations_robustness}(b). When test sequences undergo smooth temporal warping, the average Recall@1 drops by only 4.2\%. This small decrease aligns with the mathematical bounds derived earlier: a fixed patch scaffold prevents boundaries from drifting when the time axis stretches or compresses. Adding noise directly to the geometry channels (such as curvature and slope) causes almost no measurable degradation, with an average change of $-0.9\%$. Because the geometry features only condition the token representations rather than determining where tokens begin and end, noisy descriptors do not break the underlying sequence structure. In contrast, local span masking removes contiguous data segments, leading to an average 21.3\% decrease in Recall@1. This larger drop confirms that the model relies heavily on local shape features to find matches. Finally, shuffling the order of the patches only reduces performance by 1.9\%, suggesting the evaluated tasks depend more on the presence of specific local shapes than their exact global positions.

\subsection{Ablation Studies and Metric Trade-offs}
\label{subsec:ablations}

To understand how each architectural component contributes to the final performance, we test multiple ablated configurations (Figure~\ref{fig:ablations_robustness}a). The results show that the fixed scaffold must be paired with geometry conditioning to work well. Removing the curvature and residual timing inputs lowers accuracy on the ECG and speech datasets, indicating that the network uses these geometric signals to mathematically correct for local distortions.

Next, we introduce late-interaction scoring to take full advantage of the patch-level representations. Instead of comparing sequences using only a single pooled vector, late interaction (MaxSim) computes pairwise similarities among all local tokens in a query and a candidate, rewarding pairs that share at least one strongly matched local feature. As shown in Table~\ref{tab:main_retrieval}(b), applying this scoring rule as a secondary step to a top-50 candidate list increases Recall@1 by 9.3 points on NSRDB and 17.7 points on MITDB. However, this scoring method naturally favors candidates with highly similar local segments. Under a strict non-overlap evaluation, where overlapping windows from the same patient are treated as negatives, the performance of the late-interaction scorer falls sharply (Appendix~\ref{app:pareto}). Consequently, we offer late interaction as a specialized scoring step for tasks that reward local shape matching, rather than a fixed component of the main architecture.

\begin{table}[htbp]
\centering
\scriptsize
\setlength{\tabcolsep}{4.5pt}
\renewcommand{\arraystretch}{1.15}

\caption{\textbf{Main retrieval performance.} (a) reports single-vector retrieval (R@1 / mAP) under the shared nearest-neighbor protocol, comparing the GeoPatch architecture against representation-learning baselines on equal scoring footing.  (b) reports the effect of an optional late-interaction reranker (MaxSim) applied to the cached GeoPatch tokens; this is presented as a scoring-rule ablation rather than as the main architectural claim, since MaxSim is protocol-dependent (see Appendix~\ref{app:pareto}).  Bold = best in column; gray = second-best.}
\label{tab:main_retrieval}

\vspace{2.0em}
\noindent
\textbf{(a) Architecture comparison: single-vector retrieval (R@1 / mAP)}
\vspace{0.15em}

\resizebox{\linewidth}{!}{%
\begin{tabular}{lcccccc}
\toprule
& \multicolumn{3}{c}{\textbf{Rhythm ECG}} & \textbf{Speech} & \multicolumn{2}{c}{\textbf{UEA archive}} \\
\cmidrule(lr){2-4} \cmidrule(lr){5-5} \cmidrule(lr){6-7}
\textbf{Method} & \textbf{NSRDB} & \textbf{NSTDB} & \textbf{MITDB} & \textbf{CMU ARCTIC} & \textbf{UEA-mean} & \textbf{UEA-median} \\
\midrule
\textbf{GeoPatch} & \textbf{0.8800} / 0.4589 & \textbf{1.0000} / \textbf{0.9999} & 0.8018 / 0.4853 & \textbf{0.8515 / 0.6338} & 0.7639 / 0.5382 & 0.8447 / 0.5796 \\
\midrule
Vanilla Trans. & 0.6736 / \textbf{0.5023} & 0.9982 / 0.9966 & 0.6458 / 0.4853 & 0.8012 / 0.5761 & 0.5920 / 0.4484 & 0.6341 / 0.3667 \\
PatchTST & 0.6520 / 0.4943 & 1.0000 / 0.9913 & 0.6334 / 0.4413 & 0.8044 / 0.5754 & 0.7326 / 0.6232 & 0.8594 / \textbf{0.7087} \\
TS2Vec & 0.6445 / 0.4927 & 1.0000 / 0.9954 & 0.7455 / 0.5276 & 0.8076 / 0.5741 & 0.7841 / {\color{gray}0.6550} & \textbf{0.8694} / 0.6548 \\
TS-TCC & 0.6799 / 0.4882 & 1.0000 / 0.9894 & 0.7587 / 0.5149 & 0.8098 / 0.5713 & 0.7045 / 0.5624 & 0.7645 / 0.6100 \\
TS-ResNet & 0.6676 / 0.4755 & 1.0000 / {\color{gray}0.9981} & 0.7343 / 0.4835 & 0.7619 / 0.5445 & 0.7691 / 0.6524 & {\color{gray}0.8652} / 0.6513 \\
InceptionTime & 0.6609 / 0.4951 & 1.0000 / 0.9906 & 0.7690 / {\color{gray}0.5389} & 0.7913 / 0.5457 & {\color{gray}0.7925} / \textbf{0.6735} & 0.8573 / {\color{gray}0.6897} \\
MiniROCKET & 0.6673 / 0.4930 & 0.9945 / 0.9494 & \textbf{0.9138 / 0.6056} & {\color{gray}0.8450 / 0.6323} & \textbf{0.8538} / 0.6471 & 0.8507 / 0.6464 \\
InvConvNet & {\color{gray}0.6805 / 0.4998} & 1.0000 / 0.9966 & {\color{gray}0.8416} / 0.5320 & 0.7517 / 0.4919 & 0.7707 / 0.4935 & 0.6712 / 0.4584 \\
\bottomrule
\end{tabular}}

\vspace{0.6em}

\noindent
\textbf{(b) Late-interaction reranker: GeoPatch + MaxSim vs GeoPatch}
\vspace{0.15em}

\resizebox{\linewidth}{!}{%
\begin{tabular}{lcccccc}
\toprule
& \textbf{NSRDB} & \textbf{NSTDB} & \textbf{MITDB} & \textbf{CMU ARCTIC} & \textbf{UEA-mean} & \textbf{UEA-median} \\
\midrule
\multicolumn{7}{l}{\textit{Standard retrieval (R@1 / mAP)}} \\
GeoPatch (base) & 0.8800 / 0.4589 & 1.0000 / 0.9999 & 0.8018 / 0.4853 & 0.8515 / 0.6338 & 0.7639 / 0.5382 & 0.8447 / 0.5796 \\
GeoPatch + MaxSim & 0.9727 / 0.4591 & 1.0000 / 0.9999 & 0.9785 / 0.4905 & 0.8538 / 0.6835 & 0.7984 / 0.5427 & 0.8771 / 0.5807 \\
$\Delta$ R@1 / mAP & +9.27 / +0.02 & 0 / 0 & +17.67 / +0.52 & +0.23 / +4.97 & +3.45 / +0.45 & +3.24 / +0.11 \\
\midrule
\multicolumn{7}{l}{\textit{Strict non-overlap retrieval (R@1; ECG only --- on UEA / CMU strict $\equiv$ standard, see Appendix ~\ref{app:robustness})}} \\
GeoPatch (base) & 0.3229 & 0.6494 & 0.3885 & \textemdash & \textemdash & \textemdash \\
GeoPatch + MaxSim & 0.0563 & 0.1697 & 0.0355 & \textemdash & \textemdash & \textemdash \\
$\Delta$ R@1 (Strict) & -26.66 & -47.97 & -35.30 & \textemdash & \textemdash & \textemdash \\
\bottomrule
\end{tabular}}
\end{table}

\begin{figure}[htbp]
\centering

\begin{minipage}{\linewidth}
    {\raggedright \textbf{(a)} \par}
    \vspace{-0.1em}
    \centering
    \IfFileExists{geopatch_ablation_heatmap.png}{%
      \includegraphics[width=0.98\linewidth,keepaspectratio]{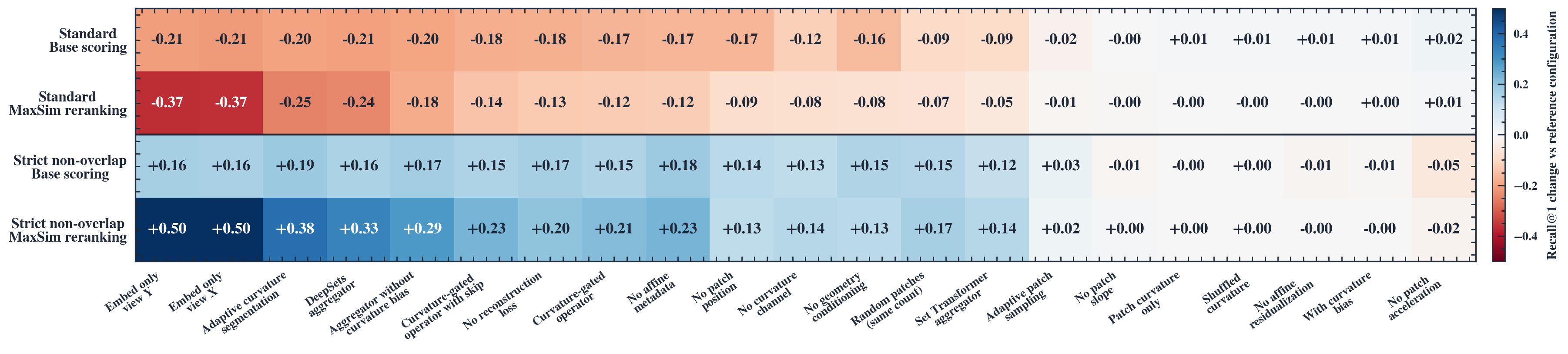}%
    }{%
      \fbox{\parbox[c][0.82in][c]{0.98\linewidth}{\centering\footnotesize Ablation heatmap placeholder}}%
    }
\end{minipage}

\vspace{0.2em}

\begin{minipage}{\linewidth}
    {\raggedright \textbf{(b)} \par}
    \vspace{-0.1em}
    \centering
    \IfFileExists{geopatch_robustness_heatmap.png}{%
      \includegraphics[width=0.98\linewidth,keepaspectratio]{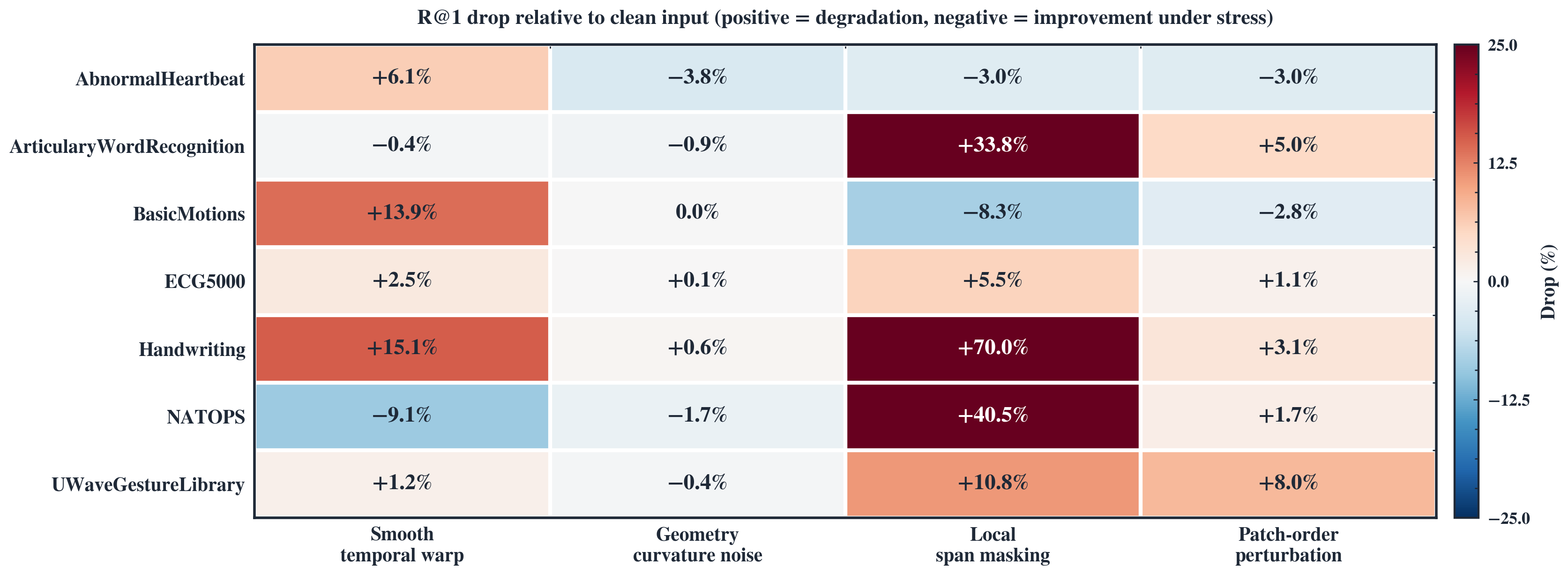}%
    }{%
      \fbox{\parbox[c][0.82in][c]{0.98\linewidth}{\centering\footnotesize Robustness heatmap placeholder}}%
    }
\end{minipage}

\caption{\textbf{Retrieval stability and metric trade-offs.} \textbf{(a)} Ablations of the fixed scaffold and geometry conditioning. Divergence between Standard and Strict protocols reveals a trade-off: optimizing for local late-interaction similarity inherently penalizes exact cross-record identity. \textbf{(b)} Resilience under synthetic perturbations. High stability under smooth warping and geometric noise (<5\% degradation) confirms boundary drift mitigation. The larger drop (10.6\%--22.9\%) under span masking reflects the expected cost of removing discriminative local morphology.}
\label{fig:ablations_robustness}
\end{figure}
\FloatBarrier
\section{Limitations and Boundary Conditions}
\label{sec:limitations}

The GeoPatch inductive bias is specifically optimized for invariance to local temporal deformations where morphology is the primary discriminative feature. Consequently, it is ill-suited for tasks where absolute temporal positioning, calibrated amplitude, or global phase relationships carry the target signal. In such regimes, geometry conditioning may introduce unnecessary model capacity without providing functional invariance.

Furthermore, our stability analysis is conditioned on the existence of a fixed scaffold and Lipschitz-continuous activations. While GeoPatch mitigates boundary drift, it does not guarantee invariance to singular or non-diffeomorphic transformations that violate these continuity assumptions. Finally, as a retrieval-first architecture, GeoPatch is intended to populate the top-K candidate set efficiently; it remains complementary to, rather than a replacement for, high-precision pairwise alignment methods (e.g., DTW) which may serve as optimal rerankers in a multi-stage deployment pipeline.

\section{Conclusion}
\label{sec:conclusion}

This work resolves a central instability in patch-based sequence modeling by decoupling token support from geometric modulation. By enforcing a fixed architectural scaffold and relegating signal geometry to continuous conditioning, the GeoPatch framework successfully transforms discrete boundary drift into bounded, stable feature variation. 

Empirical evaluations across high-variance domains—spanning cardiac rhythms to acoustic speech—demonstrate that this structural separation is strictly necessary for high-precision, early-rank retrieval. Furthermore, the efficacy of the late-interaction protocol establishes a broader principle for time-series representation: warp-robustness cannot be isolated to the encoder. It demands a rigorous alignment between a stable token interface and a scoring metric explicitly designed to preserve local morphological correspondence.

\appendix

\section{Dataset Details, Targets, Protocols, and Preprocessing}
\label{app:datasets}
The main paper groups the evaluation into ECG, speech, and UEA multivariate retrieval.  This appendix records the operational target for each family because the retrieval target determines whether local patch similarity, global identity, or cross-record matching should be rewarded.  It also specifies the preprocessing pipeline used for the UEA archive, since this pipeline determines the raw patch tensor, the geometry tensor, the patch mask, and the time mask consumed by GeoPatch.

\begin{table}[h]
\centering
\scriptsize
\setlength{\tabcolsep}{4pt}
\renewcommand{\arraystretch}{1.2}
\caption{\textbf{Dataset families and retrieval definitions.} Standard retrieval evaluates the candidate set induced by the dataset construction. Strict retrieval removes overlapping or same-record candidates when such candidates may behave as near duplicates.}
\label{tab:dataset_protocols}
\begin{tabularx}{\linewidth}{p{0.18\linewidth}p{0.21\linewidth}p{0.28\linewidth}X}
\toprule
\textbf{Family} & \textbf{Datasets} & \textbf{Retrieval target} & \textbf{Protocol concern} \\
\midrule
Rhythm ECG & NSRDB, NSTDB, MITDB & Windows with compatible rhythm morphology under rate, phase, and recording variation & Same-record and overlapping windows can inflate local similarity; strict evaluation emphasizes cross-record retrieval. \\
Speech & CMU ARCTIC & Prompt- or speaker-relevant local acoustic structure under tempo and prosody variation & Nearby utterances or repeated prompts can favor surface similarity; protocol choice determines whether this is valid evidence. \\
UEA archive & 17 multivariate datasets & Same-class retrieval for heterogeneous multivariate time series & Dataset-provided train/test splits define the gallery; heterogeneity motivates mean and median summaries. \\
\bottomrule
\end{tabularx}
\end{table}

For the UEA archive, we use the standard train/test split supplied by each dataset and treat each test sequence as a query against the test gallery with the self-match removed.  The 17 datasets are grouped as follows:
\begin{itemize}[leftmargin=1.25em,itemsep=0.1em,topsep=0.1em]
    \item \textbf{Physiological and clinical:} AbnormalHeartbeat, AtrialFibrillation, Epilepsy.
    \item \textbf{Motion and tracking:} AllGestureWiimoteX, AllGestureWiimoteY, AllGestureWiimoteZ, BasicMotions, CharacterTrajectories, Cricket, ERing, FingerMovements, HandMovementDirection.
    \item \textbf{Audio and spectrogram:} ArticularyWordRecognition, DuckDuckGeese.
    \item \textbf{Other sensor and synthetic:} EigenWorms, EthanolConcentration, FaceDetection.
\end{itemize}

\paragraph{Cached representation and manifest.}
Each UEA dataset is converted into cached NumPy stacks, \texttt{train\_windows.npy} and \texttt{val\_windows.npy}, with shape $(N,T,C)$, where $C=1$ for univariate datasets and $C>1$ for multivariate datasets.  A JSONL manifest stores the split, integer index, label string, and dataset name for each sequence.  The dataset wrapper reads the requested split, maps label strings to integer labels within that dataset, and returns the six-tuple consumed by the GeoPatch training loop:
\[
    (x_{\mathrm{raw}},\; y_{\mathrm{raw}},\; g,\; p,\; l,\; q),
\]
where $x_{\mathrm{raw}}$ is the possibly corrupted input patch tensor, $y_{\mathrm{raw}}$ is the clean target patch tensor, $g$ is the repeated patch-level geometry tensor, $p$ is the patch-validity mask, $l$ is the class label, and $q$ is the within-patch time mask.  Returning both $x_{\mathrm{raw}}$ and $y_{\mathrm{raw}}$ allows the same loader to support identity retrieval and masked-patch pretext variants without changing the model interface.

\paragraph{Normalization and geometry extraction.}
Unless otherwise stated, each sequence is standardized by per-window, per-channel $z$-scoring before patch extraction.  For a multichannel signal $s\in\mathbb R^{T\times C}$, the preprocessing computes first and second finite-difference derivatives along the time axis.  The scalar speed, acceleration norm, and curvature proxy are
\[
    v_t=\|\nabla s_t\|_2,\qquad
    a_t=\|\nabla^2 s_t\|_2,\qquad
    \kappa_t=\frac{a_t}{(1+v_t^2)^{3/2}+\varepsilon}.
\]
These descriptors are not used to choose the default patch boundaries.  They enter either as an additional raw channel, when curvature-as-channel is enabled, or as components of the conditioning vector.  We use log compression for nonnegative curvature-like summaries when enabled, which reduces the influence of isolated derivative spikes.

\paragraph{Patch construction.}
The default patching rule is a fixed sliding scaffold with patch length $\ell=16$ and at most $K=16$ patches per sequence.  Patch centers are selected in observed sample coordinates and do not depend on curvature, peaks, or learned boundary scores.  For variable-length sequences or incomplete final windows, the patch-validity mask $p_k$ marks unusable patches, which are excluded from aggregation, patch decorrelation, and optional reconstruction.  An adaptive-curvature schedule is implemented for ablations, but it is not the default mechanism used to support the main claim.  This distinction is central: curvature is used as evidence after support is fixed, not as the rule that defines support.

\paragraph{Geometry tensor.}
For each patch, the conditioning tensor contains two global descriptors, mean speed and mean curvature, repeated across all patch positions.  Depending on the configuration, additional patch-level descriptors include normalized patch center, patch index, affine metadata, mean patch curvature, mean patch slope, mean patch acceleration, and normalized segment length.  In the default retrieval setting, normalized patch position, mean slope, and mean acceleration are enabled.  Patch-level descriptors are multiplied by the validity mask so that padding cannot become an artificial geometry signal.

\paragraph{Pretext masking.}
The loader supports an identity mode and a masked-patch pretext mode.  In the identity setting, $x_{\mathrm{raw}}=y_{\mathrm{raw}}$ and the time mask is one on valid patch positions.  In the masked setting, contiguous spans inside valid patches are zeroed according to a fixed mask ratio and span length, with deterministic seeding by sample index when reproducible corruption is required.  This design makes the robustness and reconstruction settings use the same patch and geometry interface as the retrieval setting.

\section{Detailed Robustness Results}
\label{app:robustness}
\begin{figure}[h]
\centering
\IfFileExists{geopatch_stress_examples.png}{%
  \includegraphics[width=\linewidth,keepaspectratio]{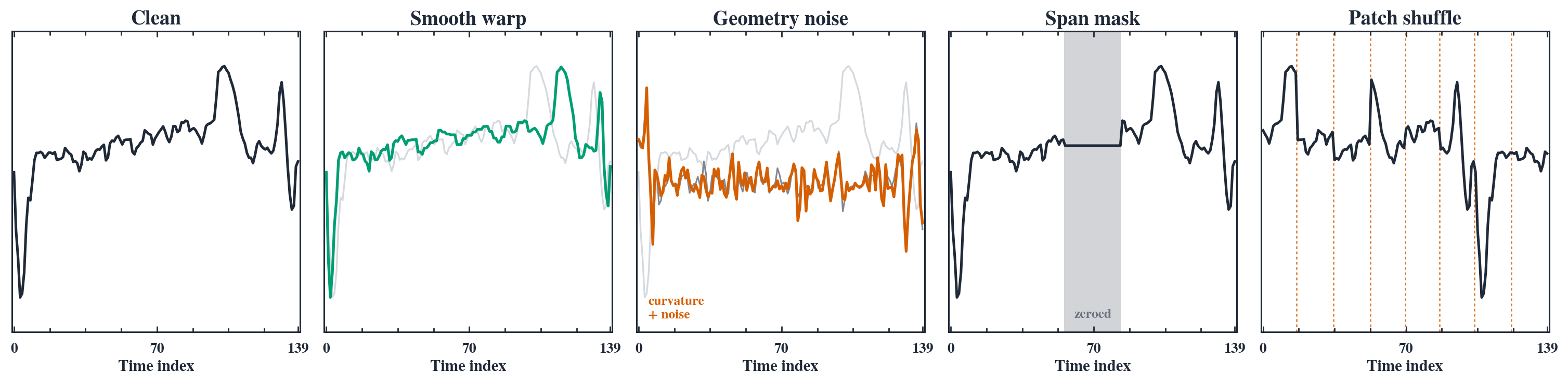}%
}{%
  \fbox{\parbox[c][1in][c]{\linewidth}{\centering\footnotesize Stress examples placeholder}}%
}
\caption{\textbf{Visual examples of the four stress conditions}, applied to a single ECG5000 sample (test index 42, $T{=}140$). From left to right: clean input, smooth temporal warp ($\sigma(t)=t+0.1T\sin(2\pi t/T)$), geometry/curvature noise applied to the derivative channel, local span masking over a $0.20T$ contiguous segment, and patch-order perturbation obtained by shuffling 8 chunks with seed 1337. The perturbations separate timing deformation, noisy conditioning, missing evidence, and disrupted patch structure.}
\label{fig:stress_examples}
\end{figure}

Table~\ref{tab:full_robustness} reports the full robustness study used to summarize Fig.~\ref{fig:ablations_robustness}(b).  We keep the clean rows in the table so that relative drops can be checked directly from the reported retrieval metrics rather than only from the heatmap.  The standard and strict columns coincide for these representative UEA datasets because the class-level test-gallery protocol does not introduce overlapping same-record candidates in the way ECG window retrieval does.

\vspace{-1.0em}

\noindent
\begin{minipage}{\textwidth}
\centering
\scriptsize
\setlength{\tabcolsep}{3.0pt}
\renewcommand{\arraystretch}{1.08}

\captionof{table}{
\textbf{Robustness study across representative UEA multivariate time-series datasets.}
Each dataset is evaluated under controlled perturbations that test temporal deformation, geometric noise, missing local evidence, and disrupted patch structure.  Cells report retrieval metrics under the shared nearest-neighbor protocol.
}
\label{tab:full_robustness}

\resizebox{\textwidth}{!}{
\begin{tabular}{
ll
cccc
cccc
}
\toprule
\multirow{2}{*}{\textbf{Dataset}}
&
\multirow{2}{*}{\textbf{Stress condition}}
&
\multicolumn{4}{c}{\textbf{Standard retrieval}}
&
\multicolumn{4}{c}{\textbf{Strict non-overlap retrieval}}
\\
\cmidrule(lr){3-6}
\cmidrule(lr){7-10}
&
&
\textbf{R@1}
&
\textbf{R@5}
&
\textbf{mAP}
&
\textbf{MRR}
&
\textbf{R@1}
&
\textbf{R@5}
&
\textbf{mAP}
&
\textbf{MRR}
\\
\midrule
\multirow{5}{*}{AbnormalHeartbeat}
& Clean input                         & 0.6439 & 0.9463 & 0.6253 & 0.7799 & 0.6439 & 0.9463 & 0.6253 & 0.7799 \\
& Smooth temporal warp                & 0.6049 & 0.9220 & 0.6121 & 0.7469 & 0.6049 & 0.9220 & 0.6121 & 0.7469 \\
& Geometry / curvature noise          & 0.6683 & 0.9463 & 0.6212 & 0.7868 & 0.6683 & 0.9463 & 0.6212 & 0.7868 \\
& Local span masking                  & 0.6634 & 0.9659 & 0.6271 & 0.7915 & 0.6634 & 0.9659 & 0.6271 & 0.7915 \\
& Patch-order perturbation            & 0.6634 & 0.9610 & 0.6252 & 0.7876 & 0.6634 & 0.9610 & 0.6252 & 0.7876 \\
\midrule
\multirow{5}{*}{ArticularyWordRecognition}
& Clean input                         & 0.7400 & 0.9033 & 0.4047 & 0.8168 & 0.7400 & 0.9033 & 0.4047 & 0.8168 \\
& Smooth temporal warp                & 0.7433 & 0.9233 & 0.4427 & 0.8285 & 0.7433 & 0.9233 & 0.4427 & 0.8285 \\
& Geometry / curvature noise          & 0.7467 & 0.9033 & 0.4082 & 0.8195 & 0.7467 & 0.9033 & 0.4082 & 0.8195 \\
& Local span masking                  & 0.4900 & 0.8033 & 0.2517 & 0.6295 & 0.4900 & 0.8033 & 0.2517 & 0.6295 \\
& Patch-order perturbation            & 0.7033 & 0.9033 & 0.3951 & 0.7909 & 0.7033 & 0.9033 & 0.3951 & 0.7909 \\
\midrule
\multirow{5}{*}{BasicMotions}
& Clean input                         & 0.9000 & 0.9750 & 0.6670 & 0.9328 & 0.9000 & 0.9750 & 0.6670 & 0.9328 \\
& Smooth temporal warp                & 0.7750 & 0.9500 & 0.6359 & 0.8460 & 0.7750 & 0.9500 & 0.6359 & 0.8460 \\
& Geometry / curvature noise          & 0.9000 & 0.9750 & 0.6720 & 0.9325 & 0.9000 & 0.9750 & 0.6720 & 0.9325 \\
& Local span masking                  & 0.9750 & 1.0000 & 0.6656 & 0.9800 & 0.9750 & 1.0000 & 0.6656 & 0.9800 \\
& Patch-order perturbation            & 0.9250 & 0.9750 & 0.6666 & 0.9456 & 0.9250 & 0.9750 & 0.6666 & 0.9456 \\
\midrule
\multirow{5}{*}{ECG5000}
& Clean input                         & 0.9184 & 0.9609 & 0.7373 & 0.9389 & 0.9184 & 0.9609 & 0.7373 & 0.9389 \\
& Smooth temporal warp                & 0.8956 & 0.9524 & 0.6947 & 0.9231 & 0.8956 & 0.9524 & 0.6947 & 0.9231 \\
& Geometry / curvature noise          & 0.9171 & 0.9598 & 0.7348 & 0.9375 & 0.9171 & 0.9598 & 0.7348 & 0.9375 \\
& Local span masking                  & 0.8678 & 0.9440 & 0.5814 & 0.9045 & 0.8678 & 0.9440 & 0.5814 & 0.9045 \\
& Patch-order perturbation            & 0.9080 & 0.9587 & 0.7286 & 0.9324 & 0.9080 & 0.9587 & 0.7286 & 0.9324 \\
\midrule
\multirow{5}{*}{Handwriting}
& Clean input                         & 0.4200 & 0.7671 & 0.1165 & 0.5732 & 0.4200 & 0.7671 & 0.1165 & 0.5732 \\
& Smooth temporal warp                & 0.3565 & 0.6788 & 0.1017 & 0.5015 & 0.3565 & 0.6788 & 0.1017 & 0.5015 \\
& Geometry / curvature noise          & 0.4176 & 0.7694 & 0.1173 & 0.5711 & 0.4176 & 0.7694 & 0.1173 & 0.5711 \\
& Local span masking                  & 0.1259 & 0.3929 & 0.0631 & 0.2573 & 0.1259 & 0.3929 & 0.0631 & 0.2573 \\
& Patch-order perturbation            & 0.4071 & 0.7588 & 0.1153 & 0.5619 & 0.4071 & 0.7588 & 0.1153 & 0.5619 \\
\midrule
\multirow{5}{*}{NATOPS}
& Clean input                         & 0.6722 & 0.9444 & 0.4533 & 0.7792 & 0.6722 & 0.9444 & 0.4533 & 0.7792 \\
& Smooth temporal warp                & 0.7333 & 0.9500 & 0.4463 & 0.8256 & 0.7333 & 0.9500 & 0.4463 & 0.8256 \\
& Geometry / curvature noise          & 0.6833 & 0.9611 & 0.4531 & 0.7880 & 0.6833 & 0.9611 & 0.4531 & 0.7880 \\
& Local span masking                  & 0.4000 & 0.8167 & 0.2260 & 0.5781 & 0.4000 & 0.8167 & 0.2260 & 0.5781 \\
& Patch-order perturbation            & 0.6611 & 0.9333 & 0.4464 & 0.7725 & 0.6611 & 0.9333 & 0.4464 & 0.7725 \\
\midrule
\multirow{5}{*}{UWaveGestureLibrary}
& Clean input                         & 0.7781 & 0.9219 & 0.3446 & 0.8396 & 0.7781 & 0.9219 & 0.3446 & 0.8396 \\
& Smooth temporal warp                & 0.7688 & 0.9219 & 0.3523 & 0.8341 & 0.7688 & 0.9219 & 0.3523 & 0.8341 \\
& Geometry / curvature noise          & 0.7812 & 0.9187 & 0.3403 & 0.8394 & 0.7812 & 0.9187 & 0.3403 & 0.8394 \\
& Local span masking                  & 0.6937 & 0.8906 & 0.3253 & 0.7820 & 0.6937 & 0.8906 & 0.3253 & 0.7820 \\
& Patch-order perturbation            & 0.7156 & 0.9187 & 0.3262 & 0.7986 & 0.7156 & 0.9187 & 0.3262 & 0.7986 \\
\midrule

\multicolumn{10}{l}{\textit{Average percentage drop relative to clean input; lower absolute drop indicates stronger robustness.}} \\
\midrule
\multicolumn{2}{l}{Smooth temporal warp}        &  4.2\% &  2.1\% &  2.2\% &  3.0\% &  4.2\% &  2.1\% &  2.2\% &  3.0\% \\
\multicolumn{2}{l}{Geometry / curvature noise}  & -0.9\% & -0.2\% &  0.0\% & -0.3\% & -0.9\% & -0.2\% &  0.0\% & -0.3\% \\
\multicolumn{2}{l}{Local span masking}          & 21.3\% & 10.6\% & 22.9\% & 15.4\% & 21.3\% & 10.6\% & 22.9\% & 15.4\% \\
\multicolumn{2}{l}{Patch-order perturbation}    &  1.9\% &  0.2\% &  1.6\% &  1.3\% &  1.9\% &  0.2\% &  1.6\% &  1.3\% \\

\bottomrule
\end{tabular}
}

\begin{minipage}{0.96\textwidth}
\footnotesize
\textbf{Stress definitions.}
Smooth temporal warp applies local stretching and compression while preserving sequence identity.
Geometry/curvature noise perturbs the geometry-conditioning channels without changing the patch scaffold.
Local span masking removes contiguous evidence from the input sequence.
Patch-order perturbation disrupts the ordering of patch tokens while preserving local patch content.
Together, these checks test whether GeoPatch remains stable under timing deformation, noisy geometry, missing evidence, and structural disruption.
\end{minipage}

\end{minipage}

\paragraph{Interpretation.}
Geometry noise leaves token support unchanged and perturbs only the descriptor pathway, so the small average changes in R@1, R@5, mAP, and MRR are consistent with the fixed-scaffold mechanism.  Smooth temporal warping changes the values inside patches, but it does not change the patch domains, which removes boundary drift.  Span masking is the most damaging condition because the relevant morphology is no longer present.  Patch-order perturbation is milder on average, suggesting that the tested retrieval settings depend more on local morphology than on exact absolute order.  Negative drops should be read as finite-sample or protocol effects, not as evidence that a perturbation is intrinsically beneficial.

\section{Optional Conditional Operator Architecture}
\label{app:operator}
Some full-training configurations include an optional conditional operator branch.  This branch is not required by the main retrieval mechanism; it is used when an auxiliary reconstruction loss is enabled.  Let $h_k$ denote the temporal content feature and $c_k$ the geometry condition. The conditional operator forms
\[
    \hat h_k^{(y)} = \Phi_\theta\left(h_k \parallel \psi_\phi(c_k)\right),
\]
where $\parallel$ denotes concatenation and $\psi_\phi$ projects the condition into the latent space of patch representations. The output is projected back to the input patch space through
\[
    \mathcal L_{\mathrm{rec}}
    =\frac{1}{|\mathcal M|}\sum_{k\in\mathcal M}
    \left\|\mathrm{Proj}_{\mathrm{out}}(\hat h_k^{(y)})-x_k^{(y)}\right\|_2^2 .
\]
This auxiliary loss encourages patch tokens to retain enough morphology to reconstruct local observations under the specified condition.  It is best viewed as a regularizer for low-data or sequence-aligned settings, not as the source of the retrieval gains reported in the main table.

\section{Empirical Characterization of the Stability Bound}
\label{app:theorem1_emp}
Theorem~\ref{thm:drift} gives a conditional upper bound parameterized by Lipschitz constants of the geometry branch, patch encoder, aggregator, and normalization.  These constants are estimated from typical perturbations on the operating set; the resulting numerical bound is therefore a scale-correct characterization rather than a tight worst-case certificate.  We sweep perturbation magnitude $\eta\in\{0.05,0.10,0.15,0.20,0.25,0.30,0.40\}$ on NSRDB and track $10{,}000$ patches per setting, reporting the median, mean, and 90th-percentile of the ratio between observed displacement and the evaluated bound.

\begin{figure}[h]
\centering
\IfFileExists{geopatch_theorem1_verification.png}{%
  \includegraphics[width=0.7\linewidth,keepaspectratio]{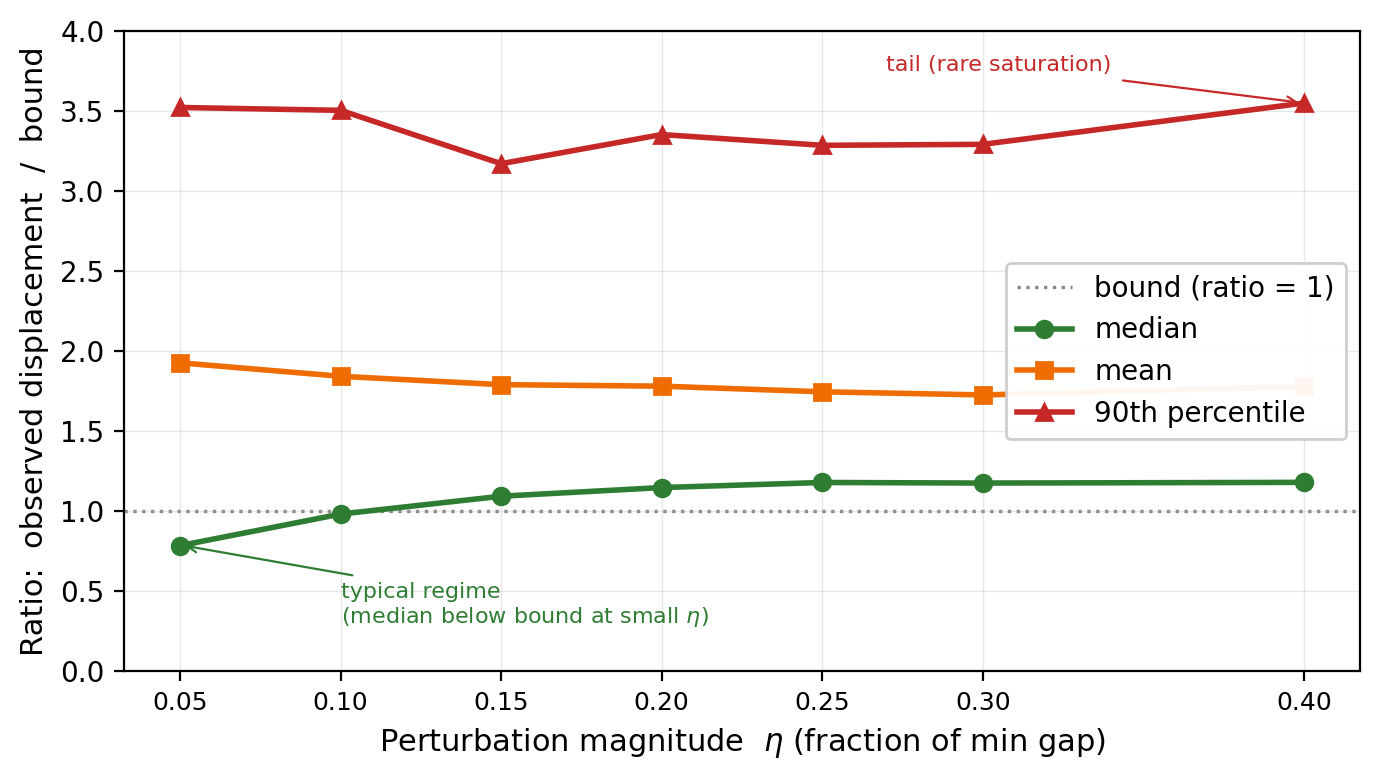}%
}{%
  \fbox{\parbox[c][1in][c]{0.7\linewidth}{\centering\footnotesize Stability-bound characterization placeholder}}%
}
\caption{\textbf{Empirical bound tightness on NSRDB.} Median, mean, and 90th-percentile ratios of observed displacement to the evaluated bound across perturbation levels.  The median characterizes the typical operating regime; the upper tail reflects rare saturation events at attention layers where local Lipschitz estimates understate the worst-case constant.}
\label{fig:theorem1_emp}
\end{figure}

The bound grows linearly with $\eta$, and the median ratio remains within a small constant factor across the swept range.  This is the correct scaling behavior for the bound under the stated assumptions; the long upper tail is expected for softmax-like aggregators where rare local saturation can transiently increase effective Lipschitz constants.  Reporting the median, mean, and 90th-percentile gives a more informative audit than asking whether every perturbation satisfies a worst-case inequality based on empirical Lipschitz estimates.

\section{Standard-vs-Strict Operating-Point Trade-off}
\label{app:pareto}
Figure~\ref{fig:pareto_appendix} plots the 22 ablations from Fig.~\ref{fig:ablations_robustness}(a) in Standard R@1 and Strict non-overlap R@1 coordinates, separately for Base scoring and MaxSim reranking.  The dashed Pareto frontier traces non-dominated operating points.  The full model in Base mode lies closer to the middle of the frontier, whereas the MaxSim version shifts toward the Standard-retrieval corner.  This is the empirical version of the monotonicity statement in Proposition~\ref{prop:maxsim}: local-surface scoring can improve one definition of retrieval while weakening another.

\begin{figure}[h]
\centering
\IfFileExists{standard_nonoverlap_pareto.png}{%
  \includegraphics[width=\linewidth,keepaspectratio]{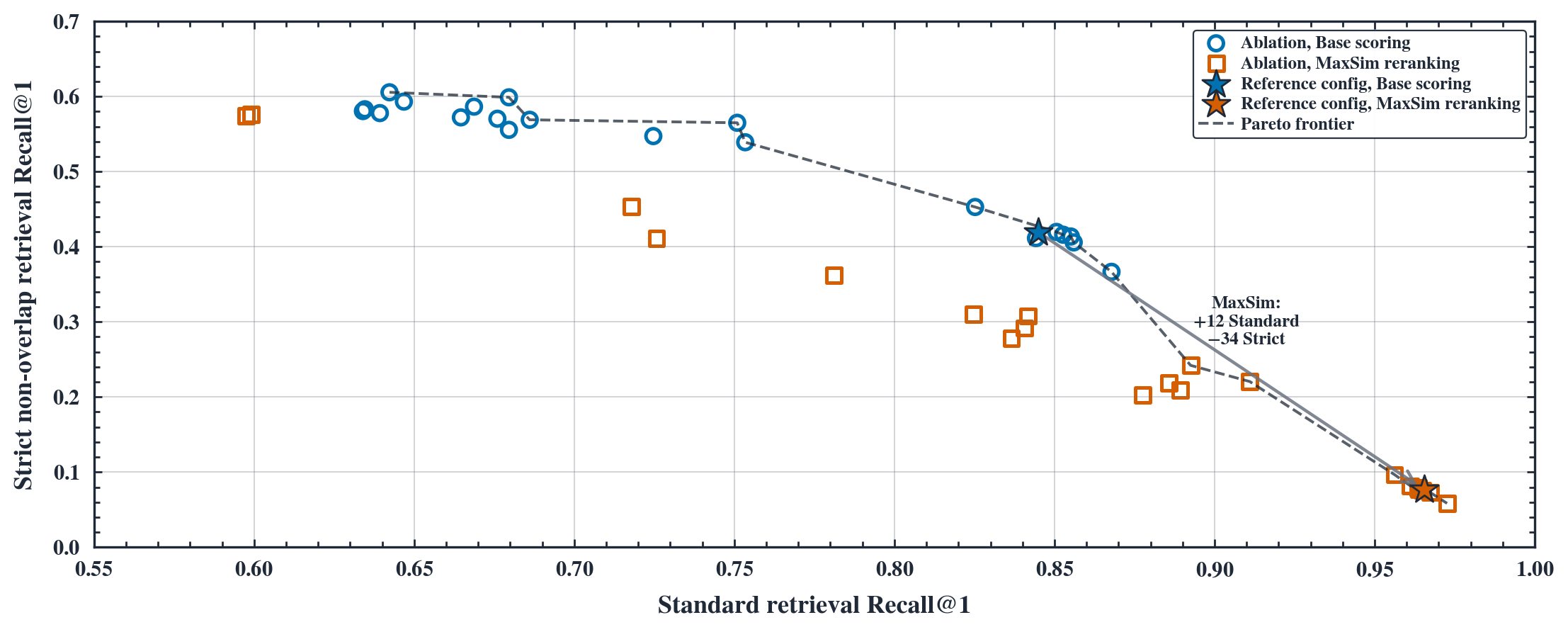}%
}{%
  \fbox{\parbox[c][1in][c]{\linewidth}{\centering \footnotesize Pareto plot placeholder}}%
}
\caption{\textbf{Standard-vs-Strict operating-point trade-off on NSRDB.} Circles denote Base scoring, squares denote MaxSim reranking, and stars denote the full configuration. The arrow between the full-model settings shows the shift induced by MaxSim. The plot should be read as an operating-point analysis, not as evidence that one scorer dominates all protocols.}
\label{fig:pareto_appendix}
\end{figure}

\section{Computational Complexity}
\label{app:complexity}
Let $A$ be the number of alignment anchors, $G$ the uniform warp grid size, $K$ the maximum number of patches, $M$ the patch length, $D$ the global embedding dimension, $d$ the per-patch embedding dimension, $N$ the number of examples evaluated, $B$ the training batch size, and $k'$ the number of reranking candidates. Warp construction sorts anchors in $O(A\log A)$, applies monotone projection/interpolation in $O(A+G)$, and computes the first and second finite differences, curvature values, and optional fixed-window smoothing in $O(G)$ time and $O(G)$ memory. Extracting the $K$ fixed-length patch views from the sampled warp/curvature arrays costs $O(KM)$ time and memory per example. Thus the default finite-difference geometry path adds $O(A\log A+G+KM)$ preprocessing per example; with fixed $G,K,M$, this is linear in the number of anchors. If the optional monotone trend-filter solver is used instead of the default PCHIP path, its extra cost is $O(TG)$, where $T$ is the number of solver iterations. For scoring, the in-batch supervised contrastive objective forms a $B\times B$ similarity matrix and costs $O(B^2D)$ time and $O(B^2)$ memory. Full retrieval over $N$ examples costs $O(N^2D)$ time for global cosine scores and $O(N^2\log N)$ time if the full ranking is materialized; top-$k'$ selection can reduce the ranking step to $O(N^2\log k')$. The implementation chunks the score matrix, so peak working memory is $O(CN)$ for chunk size $C$, plus the output ranking. The optional patch-level late interaction scores only the top $k'$ global candidates per query. A single query-candidate refinement computes a $K\times K$ patch similarity matrix in $O(K^2d)$ time and $O(K^2)$ memory, making the full reranking overhead $O(Nk'K^2d)$ time with $O(bk'K^2)$ working memory for refinement batch size $b$. Since $k'$ and $K$ are fixed hyperparameters and $k'\ll N$, the reranker adds a controlled linear-in-$N$ overhead after the global retrieval pass.

\section{Proofs}
\label{app:proofs}

\paragraph{Notation.}
Let $\mathcal P=\{\Omega_k\}_{k=1}^K$ denote a patch family, and write $V_{\mathcal P}(y)=\{y|_{\Omega_k}\}_{k=1}^K$ for the patch-value collection induced by $\mathcal P$.  We write $R(\mathcal P,y)$ for the functional that extracts $V_{\mathcal P}(y)$ from the continuous signal $y$ and then applies patch encoding and aggregation.  For GeoPatch, $\mathcal P=\mathcal S$ is fixed and only $V_{\mathcal S}(y)$ and the descriptor sequence $(g_k)_{k=1}^K$ vary with the observation.

\begin{proposition}[Boundary drift under geometry-dependent segmentation]
\label{prop:noninv}
Fix canonical content $x$ and consider content-matched observations $y_A,y_B$ generated under admissible warps.  Suppose a representation has the adaptive form in Eq.~\eqref{eq:adaptive_rep}, where $K(y)$ and $\Omega_k(y)$ are determined by a geometry-dependent rule, and suppose the aggregator is not constant over patch collections with different support, order, cardinality, or content.  If there exist content-matched observations such that $\mathcal P(y_A)\neq\mathcal P(y_B)$, then strict warp invariance cannot be guaranteed over the admissible class.
\end{proposition}

\begin{proof}
The representation depends on the observation through two arguments: the patch family $\mathcal P(y)$ and the values $V_{\mathcal P(y)}(y)$.  If $\mathcal P(y_A)\neq\mathcal P(y_B)$, then the aggregator receives either a different number of tokens, tokens defined on different supports, or tokens in a different order.  These inputs do not differ only by a small vector perturbation in a shared coordinate system.  Since the aggregator is non-constant on such collections, there exists at least one pair of content-matched observations for which the outputs differ.  Therefore invariance cannot be guaranteed solely from smoothness of the local encoder or aggregator.
\end{proof}

\begin{proposition}[Affine equivariance and normalized residual invariance]
\label{prop:affine}
Let $\tau_{\mathrm{new}}=\alpha\tau+\beta$ with $\alpha>0$ and $\beta\in\R$, and let $r(\tau)=\tau-\Proj(\tau)$ for orthogonal projection onto $\Aff$.  Then $r(\tau_{\mathrm{new}})=\alpha r(\tau)$.  If in addition $\alpha_{\mathrm{loc}}$ satisfies the affine-rescaling property $\alpha_{\mathrm{loc}}(\alpha\tau+\beta)=\alpha\,\alpha_{\mathrm{loc}}(\tau)$ for all $\alpha>0,\beta\in\R$, then $\mathcal R(\tau_{\mathrm{new}})=\mathcal R(\tau)$.
\end{proposition}

\begin{proof}
$\Aff$ is a $2$-dimensional linear subspace of $L^2([0,1])$ that contains all constant functions; in particular, $\beta\in\Aff$ for every $\beta\in\R$.  Orthogonal projection onto a closed linear subspace is itself linear, so
\[
    \Proj(\alpha\tau+\beta)=\alpha\,\Proj(\tau)+\Proj(\beta)=\alpha\,\Proj(\tau)+\beta,
\]
where the second equality uses $\Proj(\beta)=\beta$ since $\beta\in\Aff$.  Substituting,
\[
    r(\alpha\tau+\beta)=\alpha\tau+\beta-\big(\alpha\,\Proj(\tau)+\beta\big)=\alpha\,(\tau-\Proj(\tau))=\alpha\, r(\tau).
\]
The affine-rescaling property of $\alpha_{\mathrm{loc}}$ then gives
\[
    \mathcal R(\alpha\tau+\beta)=\frac{\alpha\, r(\tau)}{\alpha\,\alpha_{\mathrm{loc}}(\tau)}=\mathcal R(\tau).
    \qedhere
\]
\end{proof}

\begin{theorem}[Confidence-weighted modulation stability]
\label{thm:drift}
Using the notation of Section~\ref{sec:GeoPatch}, assume that the scaffold and valid set $\mathcal M$ are fixed under descriptor perturbations.  Assume that the geometry encoder $\gamma_\theta$ is $L_\gamma$-Lipschitz, the patch encoder $f_\theta$ is $L_f$-Lipschitz in its fused input, the aggregator $A_\theta$ is $L_A$-Lipschitz with respect to the average valid-patch norm on the bounded operating region, and the final normalization is $L_N$-Lipschitz.  Then for any per-patch descriptor perturbation $(\delta g_k)_{k=1}^K$,
\[
    \norm{T\big((g_k+\delta g_k)\big)-T\big((g_k)\big)}_2
    \le
    L_NL_AL_f L_\gamma\cdot
    \frac{1}{|\mathcal M|}\sum_{k\in\mathcal M}a_k\norm{\delta g_k}_2 .
\]
\end{theorem}

\begin{proof}
Because the scaffold and valid set are fixed, perturbing $g_k$ changes only the geometry pathway.  Let $c_k=\gamma_\theta(g_k)$ and $c_k'=\gamma_\theta(g_k+\delta g_k)$ denote the geometry-encoded vectors before and after perturbation.  By Lipschitzness of $\gamma_\theta$,
\[
    \norm{c_k'-c_k}_2\le L_\gamma\norm{\delta g_k}_2 .
\]
Let $z_k=f_\theta(\mathrm{Fuse}_\theta(h_k,c_k))$ and $z_k'=f_\theta(\mathrm{Fuse}_\theta(h_k,c_k'))$ denote the patch tokens.  The patch encoder is Lipschitz in its modulation input, and the gate $a_k$ enters multiplicatively in $\tilde z_k=m_k a_k z_k$, so
\[
    \norm{\tilde z_k'-\tilde z_k}_2\le m_k a_k L_f L_\gamma\norm{\delta g_k}_2 .
\]
Averaging over valid patches ($k\in\mathcal M$ implies $m_k=1$) gives
\[
    \frac{1}{|\mathcal M|}\sum_{k\in\mathcal M}\norm{\tilde z_k'-\tilde z_k}_2
    \le
    L_f L_\gamma\,\frac{1}{|\mathcal M|}\sum_{k\in\mathcal M}a_k\norm{\delta g_k}_2 .
\]
Applying the Lipschitz constants of $A_\theta$ and the final normalization yields the stated bound.
\end{proof}

\begin{theorem}[Top-$R$ stability under score drift]
\label{thm:topk}
Let $\hat z_q$ and $\hat z_c$ be unit-normalized query and candidate embeddings.
Suppose perturbations satisfy $\norm{\delta_q}_2\le\eta$ and $\norm{\delta_c}_2\le\eta$.  Then for any single query--candidate pair,
\[
    \abs{(\hat z_q+\delta_q)^\top(\hat z_c+\delta_c)-\hat z_q^\top\hat z_c}
    \le 2\eta+\eta^2.
\]
If a positive candidate has a score margin larger than $4\eta+2\eta^2$ above every candidate below rank $R$, then recall@$R$ is unchanged under the perturbation.
\end{theorem}

\begin{proof}
Expanding the perturbed score gives
\[
(\hat z_q+\delta_q)^\top(\hat z_c+\delta_c)-\hat z_q^\top\hat z_c
=\hat z_q^\top\delta_c+\delta_q^\top\hat z_c+\delta_q^\top\delta_c .
\]
By Cauchy--Schwarz and unit normalization, the three terms are bounded by $\eta$, $\eta$, and $\eta^2$, giving a per-pair score drift of at most $2\eta+\eta^2$.  Preserving the rank between the positive and any candidate below the cutoff requires the positive's score to remain above the negative's score after perturbation.  In the worst case the positive's score decreases by $2\eta+\eta^2$ while a lower-ranked candidate's score simultaneously increases by the same amount, so the relative margin shrinks by at most $2(2\eta+\eta^2)=4\eta+2\eta^2$.  If the initial margin exceeds this quantity, no candidate below the cutoff can overtake the positive and recall@$K$ is preserved.
\end{proof}

\begin{proposition}[Surface-similarity bias of max aggregation]
\label{prop:maxsim}
Let $S(q,c)$ be monotone non-decreasing in $M(q,c)=\max_{i,j}\cos(z_i^{(q)},z_j^{(c)})$, and assume that for a matched query $q$, the random variable $M(q,c_{\mathrm{same}})$ first-order stochastically dominates $M(q,c_{\mathrm{cross}})$, i.e.,
\[
    \Pr\!\big(M(q,c_{\mathrm{same}})>t\big) \ge \Pr\!\big(M(q,c_{\mathrm{cross}})>t\big)\quad \text{for all } t\in\R.
\]
Then $\mathbb E[S(q,c_{\mathrm{same}})] \ge \mathbb E[S(q,c_{\mathrm{cross}})]$, with strict inequality whenever the dominance is strict on a positive-measure interval where $S$ is strictly increasing.
\end{proposition}

\begin{proof}
First-order stochastic dominance combined with monotonicity of $S$ implies $\mathbb E[S\circ M_{\mathrm{same}}]\ge\mathbb E[S\circ M_{\mathrm{cross}}]$ by the standard integral characterization of stochastic dominance \citep{shaked2007stochastic}.  Strictness follows when the dominance gap $\Pr(M_s>t)-\Pr(M_x>t)$ is positive on an interval where $S$ is strictly increasing, since that gap then contributes a positive integrand to the expectation difference.  The result does not imply that MaxSim is undesirable; it states that MaxSim encodes a preference based on strongest local surface match.
\end{proof}

\begin{corollary}[No universal late-interaction scorer]
\label{cor:nflr}
A parameter-free late-interaction rule that monotonically rewards strongest local similarity cannot be expected to dominate both standard retrieval and strict non-overlap retrieval when the two protocols disagree about whether same-record local matches are valid positives.
\end{corollary}

\begin{proof}
By Proposition~\ref{prop:maxsim}, monotone max-style scoring favors the protocol in which same-record local matches are valid.  If strict retrieval excludes those matches, the same scoring preference can move valid cross-record positives down the ranking.  A single parameter-free monotone rule therefore cannot be guaranteed to optimize both protocols simultaneously.
\end{proof}

\end{document}